%% file: paper.tex
\documentclass[conference]{IEEEtran}
\pdfoutput=1

\usepackage{amsmath}
\usepackage{amsthm}
\usepackage{amssymb}
\usepackage{graphicx}
\usepackage{import}
\usepackage{xcolor}
\usepackage{color}

\usepackage{tikz}
\usetikzlibrary{patterns}

\usepackage{times}

\usepackage[numbers,sort&compress]{natbib}
\usepackage{multicol}

\usepackage{hyperref}
\hypersetup{%
  bookmarksopen,
  bookmarksnumbered,
  pdfpagemode=UseOutlines,
  colorlinks=true,
  linkcolor=blue,
  anchorcolor=blue,
  citecolor=blue,
  filecolor=blue,
  menucolor=blue,
  urlcolor=blue
}

\usepackage{cleveref}

\usepackage{subcaption}

\newcommand{\SE}[1]{\text{SE}(#1)}

\newcommand{\Xgoal}{{X_\text{goal}}}
\newcommand{\no}{{n_\text{o}}}
\newcommand{\bel}{b}
\newcommand{\Tabs}[1]{{T_{\text{#1}}}}
\newcommand{\Trelative}[2]{{{}^{#1} T_{#2}}}
\newcommand{\Tsupport}{\Tabs{\text{sup}}}

\newcommand{\sinvalid}{s_\text{invalid}}

\newcommand{\qlattice}{q_\text{lat}}
\newcommand{\Qlattice}{Q_\text{lat}}
\newcommand{\slattice}{s_\text{lat}}
\newcommand{\Slattice}{S_\text{lat}}
\newcommand{\alattice}{a_\text{lat}}
\newcommand{\Alattice}{A_\text{lat}}
\newcommand{\Tlattice}{T_\text{lat}}
\newcommand{\Rlattice}{R_\text{lat}}
\newcommand{\Omegalattice}{\Omega_\text{lat}}
\newcommand{\pilattice}{\pi_\text{lat}}
\newcommand{\Vlattice}{V_\text{lat}}

\newcommand{\Srel}{S_\text{rel}}

\newcommand{\Trel}{T_\text{rel}}
\newcommand{\Rrel}{R_\text{rel}}
\newcommand{\Vrel}{V_\text{rel}}
\newcommand{\Omegarel}{\Omega_\text{rel}}

\newcommand{\xr}{{x_\text{r}}}
\newcommand{\yr}{{y_\text{r}}}
\newcommand{\thetar}{{\theta_\text{r}}}

\newcommand{\Xr}{{X_\text{r}}}

\newcommand{\Xrlattice}{{X_\text{r,lat}}}

\newcommand{\xorel}{{x_\text{o,rel}}}
\newcommand{\xo}{{x_\text{o}}}
\newcommand{\yo}{{y_\text{o}}}
\newcommand{\thetao}{{\theta_\text{o}}}
\newcommand{\Xo}{{X_\text{o}}}

\definecolor{mdp}{HTML}{FC9272}
\definecolor{sarsop}{HTML}{A1D99B}
\definecolor{despot}{HTML}{BEAED4}
\definecolor{latdespot}{HTML}{FFD0A2}

\newcommand{\feasibleswatch}{%
  \tikz{
    \node[fill=white, inner sep=2pt, rotate=45]{};
    \draw[green, very thick](-0.15, 0)--(0.15, 0);
  }%
}
\newcommand{\infeasibleswatch}{%
  \tikz{
    \node[fill=white, inner sep=2pt, rotate=45]{};
    \draw[red, very thick](-0.15, 0)--(0.15, 0);
  }%
}
\newcommand{\unevaluatedswatch}{%
  \tikz{
    \node[fill=white, inner sep=2pt, rotate=45]{};
    \draw[lightgray, very thick](-0.15, 0)--(0.15, 0);
  }%
}

\newcommand{\mdpswatch}{%
  \tikz{\node[fill=mdp, draw=black]{};}
  \tikz{
    \node[fill=white, inner sep=2pt, rotate=45]{};
    \draw[mdp, very thick](-0.15, 0)--(0.15, 0);
  }%
}
\newcommand{\liftmdpswatch}{%
  \tikz{\node[fill=mdp, draw=black, postaction={
    pattern=north east lines,pattern color=black}]{};}
  \tikz{
    \node[fill=mdp, inner sep=2pt, rotate=45]{};
    \draw[mdp, very thick](-0.15, 0)--(0.15, 0);
  }%
}
\newcommand{\sarsopswatch}{%
  \tikz{\node[fill=sarsop, draw=black]{};}
  \tikz{
    \node[fill=white, inner sep=2pt, rotate=45]{};
    \draw[sarsop, very thick](-0.15, 0)--(0.15, 0);
  }%
}
\newcommand{\liftsarsopswatch}{%
  \tikz{\node[fill=sarsop, draw=black, postaction={
    pattern=north east lines,pattern color=black}]{};}
  \tikz{
    \node[fill=sarsop, inner sep=2pt, rotate=45]{};
    \draw[sarsop, very thick](-0.15, 0)--(0.15, 0);
  }%
}
\newcommand{\despotswatch}{%
  \tikz{\node[fill=despot, draw=black]{};}
  \tikz{
    \node[fill=white, inner sep=2pt, rotate=45]{};
    \draw[despot, very thick](-0.15, 0)--(0.15, 0);
  }%
}
\newcommand{\latdespotswatch}{%
  \tikz{\node[fill=latdespot, draw=black]{};}
  \tikz{
    \node[fill=white, inner sep=2pt, rotate=45]{};
    \draw[latdespot, very thick](-0.15, 0)--(0.15, 0);
  }%
}

\DeclareMathOperator*{\argmax}{\text{arg\,max}}
\DeclareMathOperator{\Rot}{\text{Rot}}
\DeclareMathOperator{\Trans}{\text{Trans}}
\DeclareMathOperator{\Project}{\text{Proj}}
\DeclareMathOperator{\diag}{\text{diag}}
\DeclareMathOperator{\Prob}{\text{Pr}}
\DeclareMathOperator{\uniform}{\text{uniform}}

\newtheorem{theorem}{Theorem}

\graphicspath{{figures/}}

\usepackage{ifthen,version}
\newboolean{include-notes}
\DeclareRobustCommand{\ssnote}[1]{\ifthenelse{\boolean{include-notes}}%
 {\textcolor{red}{\bf SS: #1}}{}}
\DeclareRobustCommand{\mknote}[1]{\ifthenelse{\boolean{include-notes}}%
 {\textcolor{blue}{\bf MK: #1}}{}}

\begin{document}

\title{\LARGE Configuration Lattices for Planar Contact Manipulation Under Uncertainty}
\author{
  Michael C. Koval$^*$,
  David Hsu$^\dagger$,
  Nancy S. Pollard$^*$, and
  Siddhartha S. Srinivasa$^*$ \\
  {\small%
  \href{mailto:mkoval@cs.cmu.edu}{\tt mkoval@cs.cmu.edu},
  \href{mailto:dyhsu@comp.nus.edu.sg}{\tt dyhsu@comp.nus.edu.sg},
  \{%
    \href{mailto:nsp@cs.cmu.edu}{\tt nsp},
    \href{mailto:siddh@cs.cmu.edu}{\tt siddh}%
  \}\texttt{@cs.cmu.edu}
  } \\
  $^*$The Robotics Institute, Carnegie Mellon University \\
  $^\dagger$Department of Computer Science, National University of Singapore \\
}

\maketitle
\IEEEpeerreviewmaketitle

\begin{abstract}
  This work addresses the challenge of a robot using real-time feedback from
  contact sensors to reliably manipulate a movable object on a cluttered
  tabletop. We formulate contact manipulation as a partially observable Markov
  decision process (POMDP) in the joint space of robot configurations and
  object poses. The POMDP formulation enables the robot to actively gather
  information and reduce the uncertainty on the object pose. Further, it
  incorporates all major constraints for robot manipulation: kinematic
  reachability, self-collision, and collision with obstacles. To solve the
  POMDP, we apply DESPOT, a state-of-the-art online POMDP algorithm.  Our
  approach leverages two key ideas for computational efficiency. First, it
  performs lazy construction of a configuration-space lattice by interleaving
  construction of the lattice and online POMDP planning. Second, it combines
  online and offline POMDP planning by solving relaxed POMDP offline and using
  the solution to guide the online search algorithm. We evaluated the proposed
  approach on a seven degree-of-freedom robot arm in simulation environments.
  It significantly outperforms several existing algorithms, including some
  commonly used heuristics for contact manipulation under uncertainty.
\end{abstract}

\section{Introduction}
\label{sec:intro}
Our goal is to enable robots to use real-time feedback from contact sensors to
reliably manipulate their environment. Contact sensing is an ideal source of
feedback for a manipulator because contact is intimately linked with
manipulation: the sense of touch is unaffected by occlusion and directly
observes the forces that the robot imparts on its environment.

However, contact sensors suffer from a key limitation: they provide rich
information while in contact with an object, but little information otherwise.
Fully utilizing contact sensors requires \emph{active sensing}. Consider the
robot shown in \cref{fig:intro} that is trying to push a bottle into its palm.
The robot is uncertain about the initial pose of the bottle, has only an
approximate model of physics, and receives feedback from noisy contact sensors
on its fingertips. The robot executes an \emph{information-gathering action} by
moving its hand laterally to force the bottle into one of its contact sensors.
Once the bottle is localized, it moves to achieve the goal.

In this paper, our goal is to autonomously generate policies that allow a robot
to use real-time feedback from contact sensors to reliably manipulate objects
under uncertainty. We specifically consider the problem of quasistatic
tabletop manipulation via pushing~\cite{lynch1992manipulation}.  The robot's
goal is to push a movable object into a hand-relative goal region while
avoiding collision with stationary obstacles in the environment.

We formulate the \emph{planar contact manipulation problem} as a
\emph{partially observable Markov decision process}
(POMDP)~\cite{smallwood1973optimal,kaelbling1998planning} with a reward
function that drives the robot towards the goal. Prior work has successfully
used an offline POMDP solver~\cite{kurniawati2008sarsop} to find robust,
closed-loop policies for manipulating an object relative to the
hand~\cite{horowitz2013interactive,koval2015precontact_ijrr}.

Unfortunately, these hand-relative policies perform poorly when executed on a
robotic manipulator because of limited reachability, collision with the
environment, and other types of \emph{kinematic constraints}. These constraints
occur frequently during execution, but cannot be represented in a
\emph{hand-relative POMDP} (Rel-POMDP). In this paper, we address this
limitation by planning in a POMDP that includes both the fully-observable
configuration of the manipulator and the partially-observable pose of the
object.

Finding a near-optimal policy for this POMDP is challenging for two reasons.
First, a manipulator's configuration space is continuous and high-dimensional,
typically having at least six dimensions. Second, it is difficult to perform
pre-computation because the optimal policy may dramatically change when an
obstacle is added to, removed from, or moved within the environment.

\begin{figure}[t]%
  \centering%
  \begin{minipage}{0.5\columnwidth}%
    \begin{subfigure}[t]{\textwidth}%
      \includegraphics[width=\textwidth]{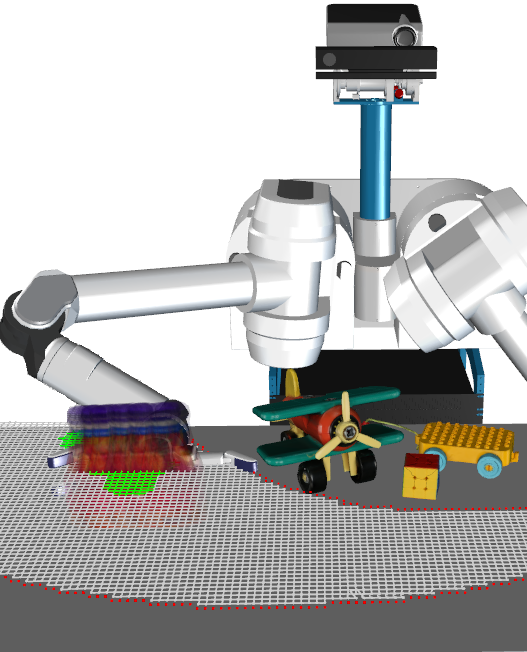}%
      \caption{Robot and Environment}%
      \label{fig:intro-all}%
    \end{subfigure}%
  \end{minipage}%
  \begin{minipage}{0.5\columnwidth}%
    \vspace{2ex}%
    \begin{subfigure}[t]{\textwidth}%
      \includegraphics[width=\textwidth]{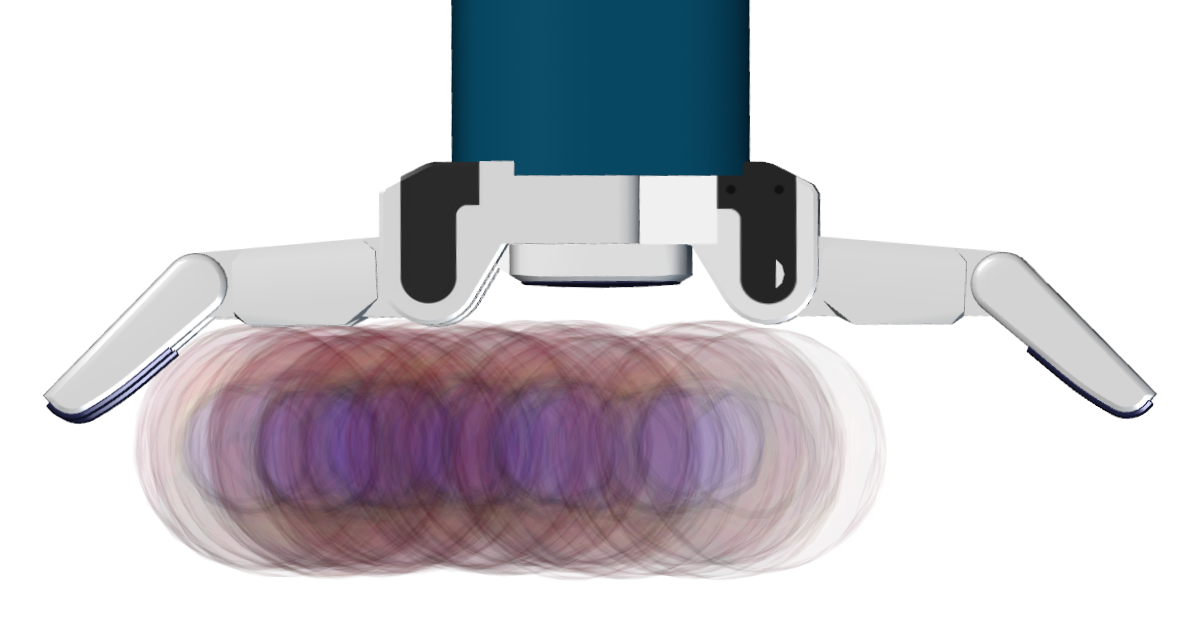}%
      \caption{Hand-Relative POMDP}%
        \label{fig:intro-relative}%
    \end{subfigure} \\[1.1ex]
    \begin{subfigure}[t]{\textwidth}%
      \includegraphics[width=\textwidth]{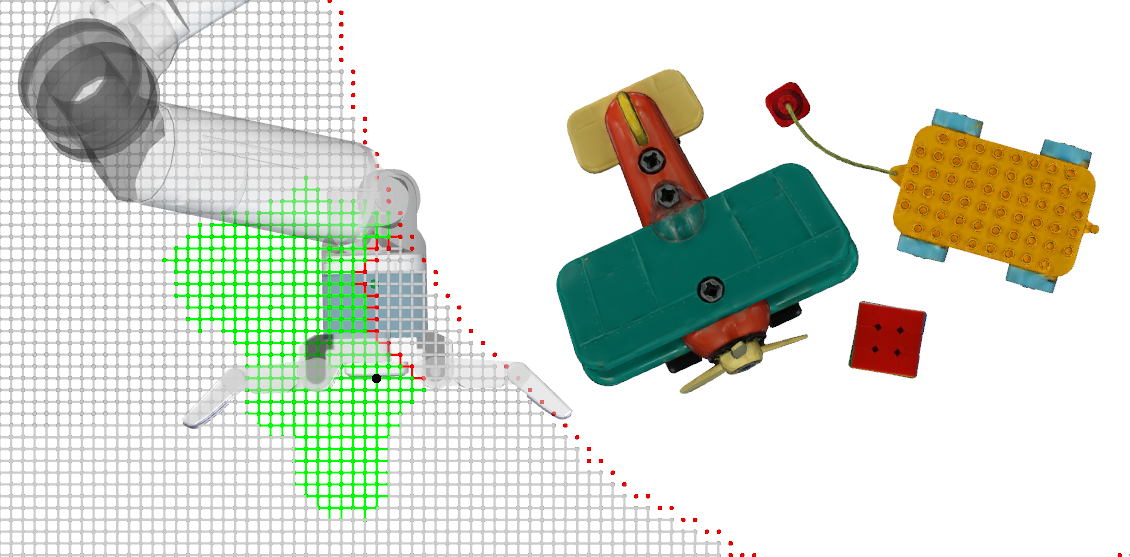}%
      \caption{Configuration Lattice}%
        \label{fig:intro-lattice}%
    \end{subfigure}%
  \end{minipage}%
  \caption[]{%
    A robot
    (\subref{fig:intro-all})~uses real-time feedback from contact
    sensors to manipulate an bottle on a cluttered table.
    (\subref{fig:intro-relative})~Our approach uses policies from the obstacle
    free, hand-relative contact manipulation problem to a search in the full
    state space.
    (\subref{fig:intro-lattice})~The configuration of the robot is represented
    as a point in a state-space lattice consisting of feasible
    (\feasibleswatch), infeasible (\infeasibleswatch), and unevaluated
    (\unevaluatedswatch) edges.  The probability distribution over the pose of
    the bottle is shown as a collection of semi-transparent renders. Best
    viewed in color.
  }%
  \label{fig:intro}%
\end{figure}

\vspace{1ex}

We leverage two key insights to overcome those challenges:
\begin{enumerate}
  \item We define a \emph{configuration lattice PODMP} (Lat-POMDP) that models
    both the configuration of the robot and the pose of the object. Lat-POMDP
    lazily constructs~\cite{pivtoraiko2005efficient,cohen2013single} a lattice
    in the robot's configuration space and uses the lattice to guarantee that
    the optimal policy does not take infeasible actions (\cref{sec:lattice}).

  \item We prove that the optimal value function of Rel-POMDP is an upper bound
    on the optimal value function of Lat-POMDP. We use this fact to create a
    heuristic that guides DESPOT~\cite{somani2013despot}, a state-of-the-art
    online POMDP solver, to quickly find a near-optimal policy for Lat-POMDP
    with no environment-specific offline pre-computation (\cref{sec:planner}).
\end{enumerate}

We validate the efficacy of the proposed algorithm on a simulation of
HERB~\cite{srinivasa2012herb}, a mobile manipulator equipped with a 7-DOF
Barrett WAM arm~\cite{salisbury1988preliminary} and the
BarrettHand~\cite{townsend2000barretthand} end-effector, manipulating an object
on a cluttered tabletop (\cref{fig:intro}, \cref{sec:simulation}). First, we
confirm that DESPOT successfully uses information-gathering actions to achieve
good performance in the absence of kinematic constraints. Then, we show that
the proposed algorithm outperforms five baselines in cluttered environments.

The proposed algorithm demonstrates that it is possible for an online POMDP
planner to simultaneously reason about object pose uncertainty and kinematic
constraints in contact manipulation. We are excited about the prospect of
exploiting the scalability of recent developments in online POMDP solvers, like
DESPOT, to solve more complex tasks in future work (\cref{sec:discussion}).

\section{Related Work}
Our work builds on a long history of research on enabling robotics reliably
manipulating objects under uncertainty. Early work focused planning open-loop
trajectories that can successfully reconfigure an object despite
non-deterministic uncertainty~\cite{lavalle1998objective} in its initial
pose~\cite{erdmann1988exploration,brokowski93curvedfence}.
More recently, the same type of worst-case, non-deterministic analysis has been
used to plan robust open-loop trajectories for
grasping~\cite{dogar2010pushgrasp,dogar2012physics} and rearrangement
planning~\cite{dogar2012planning,koval2015robust} under the quasistatic
assumption~\cite{mason1986mechanics,lynch1992manipulation}. Our approach also
makes the quasistatic assumption, but differs in two important ways: it (1)
considers probabilistic uncertainty~\cite{lavalle1998objective} and (2)
produces a closed-loop policy that uses real-time feedback from contact
sensors.

Another line of research aims to incorporate feedback from contact sensors into
manipulator control policies. This approach have been successfully used to
optimize the quality of a grasp~\cite{platt2010null} or servo towards a desired
contact sensor observation~\cite{zhang2000control,li2013control}. It is also
possible to directly learn a feedback policy that is robust to
uncertainty~\cite{pastor2011online,stulp2011learning}. These approaches have
achieved impressive real-time performance, but require a higher-level planner
to specify the goal.

One common approach is to plan a sequence of move-until-touch actions that
localize an object, then execute an open-loop trajectory to complete the
task~\cite{petrovskaya2011global,hebert2013next,javdani2013efficient}. Other
approaches, like our own, formulate the problem as a
POMDP~\cite{smallwood1973optimal,kaelbling1998planning} and find a policy that
only takes information-gathering actions when they are necessary to achieve the
goal~\cite{hsiao2007grasping,hsiao2008robust,hsiao2009relatively}.

Unfortunately, most of this work assumes that the end-effector can move freely
in the workspace and that objects do not significantly move when touched. More
recent approaches have relaxed the latter assumption by incorporating a
stochastic physics model into the
POMDP~\cite{horowitz2013interactive,koval2015precontact_ijrr} and using
SARSOP~\cite{kurniawati2008sarsop}, an offline
point-based~\cite{pineau2003point} POMDP solver, to find a near-optimal policy
for manipulating an object relative to the hand. Unfortunately, hand-relative
policies often fail when executed on a manipulator due to kinematic constraints
or collision with obstacles. We use a hand-relative policy to guide
DESPOT~\cite{somani2013despot}, an online POMDP solver~\cite{ross2008online},
in Lat-POMDP, a mixed-observable model~\cite{ong2009pomdps} that includes these
constraints.

Lat-POMDP represents the robot's configuration space as a state
lattice~\cite{pivtoraiko2005efficient}, a concept that we borrow from mobile
robot navigation~\cite{likhachev2009planning} and search-based planning for
manipulators~\cite{cohen2013single}. Similar to these approaches, and many
randomized motion planners~\cite{bohlin2000path,hauser2015lazy}, we use lazy
evaluation to defer collision checking until an action is queried by the
planner.

\section{Problem Formulation}
\label{sec:formulation}

\begin{figure}[t]%
  \centering%
  \begin{minipage}[b]{0.5\columnwidth}%
    \begin{subfigure}[t]{\textwidth}%
      \centering%
      \resizebox{\textwidth}{!}{%
        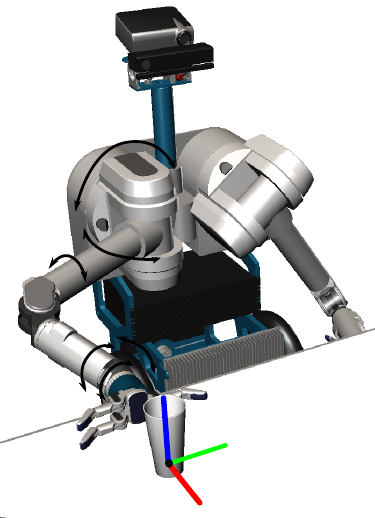%
      }%
      \caption{State, $a = (q, \xr)$}%
      \label{fig:formulation-state}%
    \end{subfigure}%
  \end{minipage}%
  \begin{minipage}[b]{0.5\columnwidth}%
    \begin{subfigure}[t]{\textwidth}%
      \includegraphics[width=\textwidth]{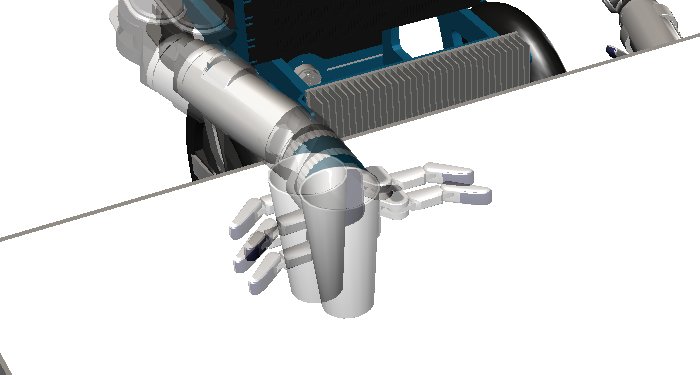}%
      \caption{Action, $a = (\xi, T)$}%
      \label{fig:formulation-action}%
    \end{subfigure} \\[1ex]
    \begin{subfigure}[t]{\textwidth}%
      \includegraphics[width=\textwidth]{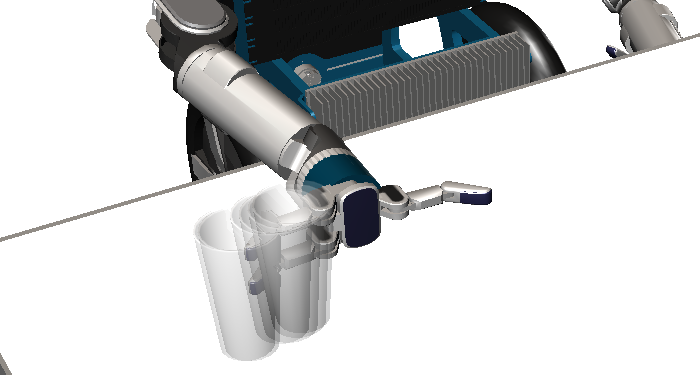}%
      \caption{Observation, $o$}%
      \label{fig:formulation-observation}%
    \end{subfigure}%
  \end{minipage}%
  \caption{
    We formulate planar contact manipulation as a POMDP with a state space
    (\subref{fig:formulation-state})~that contains the robot configuration $q$
    and the pose of the movable object $\xr$. The environment contains one
    movable object (white glass) and static obstacles.
    (\subref{fig:formulation-action})~action is a short joint-space trajectory
    $\xi$ of duration $T$. After executing an action, the robot receives
    feedback from (\subref{fig:formulation-observation})~binary contact sensors
    on its end effector.
  }%
  \label{fig:formulation}%
\end{figure}

We formulate planar contact manipulation as a partially observable Markov
decision process (POMDP)~\cite{smallwood1973optimal}. A POMDP is a tuple $(S,
A, O, T, \Omega, R)$ where $S$ is the set of states, $A$ is the set of actions,
$O$ is the set of observations, $T(s, a, s') = p(s' | s, a)$ is the transition
model, $\Omega(s, a, o) = p(o | s, a)$ is the observation model, and $R(s, a):
S \times A \to \mathbb{R}$ is the reward function.

The robot does not know the true state $s_t$. Instead, the robot tracks the
\emph{belief state} $\bel(s_t) = p(s_t | a_{1:t}, o_{1:t})$, which is a
probability distribution over the current state $s_t$ conditioned on the
history of previous actions $a_{1:t} = \{ a_1, \dotsc, a_t \}$ and observations
$o_{1:t} = \{ o_1, \dotsc, o_t \}$. The set of all belief states is known as
\emph{belief space} $\Delta$.

Our goal is to find a policy $\pi: \Delta \to A$ over belief space that
optimizes the \emph{value function}
\begin{align*}
  V^\pi \left[ \bel \right]
    &= E \left[ \sum_{t = 0}^\infty \gamma^t R(s_t, a_t) \right]
\end{align*}
where the expectation $E[\cdot]$ is taken over the sequence of states visited
by $\pi$. We use $V^*$ to denote the \emph{optimal value function}, the value
function of an optimal policy. The \emph{discount factor} $\gamma \in [0, 1)$
adjusts the relative value of present versus future reward.

In our problem, \emph{state} $s = (q, \xo) \in S$ is the configuration of the
robot $q \in Q$ and the pose of the movable object $\xo \in \Xo = \SE{3}$
(\cref{fig:formulation-state}). An \emph{action} $a = (\xi, T) \in A$ is a
trajectory $\xi: [ 0, T ] \to Q$ that starts in configuration $\xi(0)$ and ends
configuration $\xi(T)$ at time $T$ (\cref{fig:formulation-action}). We assume
that uncertainty over the pose of the object dominates controller and
proprioceptive error. Therefore, we treat $q$ as a fully-observable state
variable and neglect the dynamics of the manipulator; i.e. model it as being
position controlled.

The robot executes a quasistatic, position controlled push if it comes in
contact with the movable object. The \emph{quasistatic assumption} states that
friction is high enough to neglect the acceleration of the object; i.e. the
object stops moving as soon as it leaves contact~\cite{mason1986mechanics}.
This approximation is accurate for many tabletop  manipulation
tasks~\cite{dogar2010pushgrasp,dogar2012planning}.

We define the stochastic \emph{transition model} $T(s, a, s')$ in terms of a
deterministic quasistatic physics model~\cite{lynch1992manipulation} by
introducing noise into the parameters~\cite{duff2010motion}. We do not attempt
to refine our estimate of these parameters during execution. After executing an
action, the robot receives an \emph{observation} $o \in \{ 0, 1 \}^\no = O$
from its $\no$ binary contact sensors (\cref{fig:formulation-observation}). We
assume that a stochastic \emph{observation model} $\Omega(s', a, o)$ is
available, but make no assumptions about its form.

The robot's goal is to push the movable object into a hand-relative \emph{goal
region} $\Xgoal \subseteq X$. We encode this in the \emph{reward function}
\begin{align*}
  R(s, a) &= \begin{cases}
    \hphantom{-}0 &: [\Tabs{ee}(q)]^{-1} \xo \in \Xgoal \\
    -1            &: \text{otherwise}
  \end{cases}
\end{align*}
where $\Tabs{\text{ee}}: Q \to \SE{3}$ is the forward kinematics of the end
effector. The reward function penalizes states where the movable object is
outside of $\Xgoal$. Note that the choice of $-1$ reward is arbitrary: any
negative reward would suffice.

\section{Configuration Lattice POMDP}
\label{sec:lattice}
Solving the POMDP formulated in \cref{sec:formulation} is challenging because
the state space is continuous, the action space is infinite dimensional, the
transition model is expensive to evaluate, and the optimal policy may perform 
long-horizon information gathering.

In this section, we simplify the problem by constraining the end effector to a
fixed transformation relative to the support surface (\cref{sec:lattice-state})
and build a lattice in configuration space (\cref{sec:lattice-lattice}).
Configurations in the lattice are connected by action templates that start and
end on lattice points (\cref{sec:lattice-action}). Finally, we define a
configuration lattice POMDP that penalizes infeasible actions to insure that
the optimal policy will never take an infeasible action
(\cref{sec:lattice-feasibility}).

\subsection{Planar Constraint}
\label{sec:lattice-state}
We assume that the robot's end effector is constrained to have a fixed
transformation $\Trelative{\text{sup}}{\text{ee}} \in \SE{3}$ relative to the
support surface $\Tsupport \in \SE{3}$. A configuration $q$ satisfies this
constraint iff
\begin{align}
  \Tabs{\text{ee}}(q) &= \Tsupport \Trelative{\text{sup}}{\text{ee}}
    \Rot(\thetar, \hat{e}_z) \Trans(\xr, \yr, 0)
  \label{eqn:fk}
\end{align}
where $(\xr, \yr, \thetar) \in \Xr = SE(2)$ is the pose of the end effector in
the plane. $\Rot(\theta, \hat{v})$ denotes a rotation about axis $\hat{v}$ by
angle $\theta$ and $\Trans(v)$ denotes a translation by vector $v$.

We also assume that the movable object also has a fixed transformation
$\Trelative{\text{sup}}{\text{o}} \in SE(3)$ relative to the support surface;
i.e. that it does not tip or topple. We parameterize its pose as
\begin{align*}
  x_\text{o} = \Tsupport \Trelative{\text{sup}}{\text{o}}
    \Trans(\xo, \yo, 0) \Rot(\thetao, \hat{e}_z)
\end{align*}
where $(\xo, \yo, \thetao) \in \Xo = \SE{2}$ is its pose in the plane.

\subsection{Configuration Lattice}
\label{sec:lattice-lattice}
We discretize the space of the end effector poses $\Xr$ by constructing a
\emph{state lattice} $\Xrlattice \subseteq \Xr$ with a translational resolution
of $\Delta \xr, \Delta \yr \in \mathbb{R}^+$ and an angular resolution of
$\Delta \thetar = 2 \pi / n_\theta$ for some integer value of $n_\theta \in
\mathbb{N}$~\cite{pivtoraiko2005efficient}. The lattice consists of the
discrete set of points
\begin{align*}
  \Xrlattice &= \{ (i_x \Delta \xr, i_y \Delta \yr, i_\theta \Delta \thetar)
    : i_x, i_y, i_\theta \in \mathbb{Z} \}.
\end{align*}

Each \emph{lattice point} $\xr \in \Xrlattice$ may be reachable from multiple
configurations. We assume that we have access to an \emph{inverse kinematics
function} $\qlattice(\xr)$ that returns a single solution $\{ q \}$  that
satisfies $\Tabs{\text{ee}}(q) = \xr$ or $\emptyset$ if no such solution
exists. A solution may not exist if $\xr$ is not reachable, the end effector is
in collision, or the robot collides with the environment in all possible
inverse kinematic solutions.

Instead of planning in $S$, we restrict ourselves to the state space $\Slattice
= \Qlattice \times \Xo$ where $\Qlattice = \bigcup_{\xr \in \Xrlattice}
\qlattice(\xr)$ is the discrete set of configurations returned by
$\qlattice(\cdot)$ on all lattice points. Note that, despite the structure of
the Cartesian lattice, planning occurs in configuration space $Q$, not the task
space $\Xr$.

\subsection{Action Templates}
\label{sec:lattice-action}
Most actions do not transition between states in the lattice $\Slattice$.
Therefore, we restrict ourselves to action that are instantiated from one of a
finite set $\Alattice$ of action templates. An action template $\alattice =
(\xi_\text{lat}, T) \in \Alattice$ is a Cartesian trajectory $\xi_\text{lat}:
[0, T] \to SE(3)$ that specifies the relative motion of the end effector. The
template starts at the origin $\xi_\text{lat}(0) = I$ and ends at some lattice
point $\xi_\text{lat}(T) \in \Xrlattice$. It is acceptable for multiple actions
templates in $\Alattice$ to end at the same lattice point or have different
durations.

An action $a = (\xi, T) \in A$ satisfies template $\alattice$ at lattice point
$\xr \in \Xrlattice$ if it satisfies three conditions:
\begin{enumerate}
\item starts in configuration $\xi(0) = \qlattice(\xr)$
\item ends in configuration $\xi(T) = \qlattice( \xr \xi_\text{lat}(T) )$
\item satisfies $\Tabs{ee}(\xi(\tau)) = \xr \xi_\text{lat}(\tau)$ for all $0
  \le \tau \le T$.
\end{enumerate}
These conditions are satisfied if $\xi$ moves between two configurations in
$\Qlattice$ and produces the same end effector motion as the Cartesian
trajectory $\xi_\text{lat}$.

We define the function $\Project(\xr, a) \mapsto \alattice$ to map an action
$a$ to the template $\alattice$ it instantiates. The pre-image
$\Project^{-1}(\xr, \alattice)$ contains the set of all possible instantiations
of action template $\alattice$ at $\xr$. We assume that we have access to a
\emph{local planner} $\phi(q, \alattice)$ that returns a singleton action $\{ a
\} \subseteq \Project^{-1}(q, \alattice)$ from this set or $\emptyset$ to
indicate failure. The local planner may fail due to kinematic constraints, end
effector collision, or robot collisions.

\subsection{Lattice POMDP}
We use the lattice to define a \emph{configuration lattice POMDP}, Lat-POMDP,
$(\Slattice, \Alattice, O, \Tlattice, \Omegalattice, \Rlattice)$ with state
space $\Slattice = \Xrlattice \times \Xo$. The structure of the lattice
guarantees that that all instantiations $\Project^{-1}(q, \alattice)$ of the
action template $\alattice$ execute the same Cartesian motion $\xi_\text{lat}$
of the end effector. This motion is independent of the starting pose of the end
effector $\xr$ and configuration $\qlattice(\xr)$ of the robot.

If the movable object only contacts the end effector---not other parts of the
robot or the environment---then the motion of the object is also independent of
these variables. We refer to a violation of this assumption as
\emph{un-modelled contact}. The lattice transition model $\Tlattice(\slattice,
\alattice, s'_\text{lat})$ is identical to $T(s, a, s')$ when $\alattice$ is
feasible and no un-modelled contact occurs. If either condition is violated,
the robot deterministically transitions to an absorbing state $\sinvalid$.
Similarly, the lattice observation model $\Omegalattice(\slattice, \alattice,
o)$ is identical to $\Omega(s, a, o)$ for valid states and is uniform over $O$
for $\slattice = \sinvalid$.

We penalize invalid states, infeasible actions, and un-modelled contact in the
reward function
\begin{align*}
  \Rlattice(\slattice, \alattice) &= \begin{cases}
    -1 &: \slattice = \sinvalid \\
    -1 &: \phi(\qlattice(\xr), \alattice) = \emptyset \\
    -1 &: \xo \not\in X_\text{g} \\
    \hphantom{-}0 &: \text{otherwise}
  \end{cases}
\end{align*}
by assigning $\min_{s \in S, a \in A} R(s, a) = -1$ reward to them.

\subsection{Optimal Policy Feasibility}
\label{sec:lattice-feasibility}
Our formulation of Lat-POMDP guarantees that an optimal policy $\pilattice^*$
will never take an infeasible action:

\begin{theorem}
  An optimal policy $\pilattice^*$ of Lat-POMDP will not execute an infeasible
  action in belief $b$ if $\Vlattice^{*}[b] > \frac{-1}{1 - \gamma}$.
  \label{thm:infeasible}
\end{theorem}
\begin{proof}
  Suppose $\Vlattice^{*}[b] > \frac{-1}{1 - \gamma}$ and an optimal policy
  $\pilattice$ executes the invalid action in belief state $b$. The robot
  receives a reward of $-1$ and transitions to $\sinvalid$. For all time after
  that, regardless of the actions that $\pilattice$ takes, the robot receives a
  reward of $\Rlattice(\sinvalid, \cdot) = -1$ at each timestep. This yields a
  total reward of $\Vlattice^{\pilattice} = \frac{-1}{1 - \gamma}$, which is
  the minimum reward possible to achieve.

  The value function of the optimal policy satisfies the Bellman equation
  $\Vlattice^*[b] = \argmax_{\alattice \in \Alattice} Q^{*}[b, \alattice]$,
  where $Q^*[b, a]$ denotes the value of taking action $a$ in belief state $b$,
  then following the optimal policy for all time. This contradicts the fact
  that $\Vlattice^*[b] > \frac{-1}{1 - \gamma}$ and $\Vlattice^{\pilattice}[b]
  = \frac{-1}{1 - \gamma}$. Therefore, $\pilattice$ must not be the optimal
  policy.
\end{proof}

We can strengthen our claim if (1) we guarantee that every lattice point
reachable from the $q_0$ has at least one feasible action and (2) it is
possible to achieve the goal with non-zero probability. Under those
assumptions we know that $\Vlattice^{*}[b] > \frac{-1}{1 - \gamma}$ and
\cref{thm:infeasible} guarantees that $\pilattice^*$ will never take an
infeasible action. One simple way to satisfy condition (1) is to require that
all actions are reversible; i.e. make the lattice an undirected graph.

\section{Online POMDP Planner}
\label{sec:planner}
Lat-POMDP has a reward function that changes whenever obstacles are added to,
removed from, or moved within the environment. An offline POMDP solver, like
point-based value iteration~\cite{pineau2003point} or
SARSOP~\cite{kurniawati2008sarsop} would require re-computing the optimal
policy for each problem instance.

Instead, we use DESPOT, a state-of-the-art online POMDP solver that uses
regularization to balance the size of the policy against its
quality~\cite{somani2013despot}. DESPOT incrementally explores the
action-observation tree rooted at $\bel(s_0)$ by performing a series of trials.
Each \emph{trial} starts at the root node, descends the tree, and terminates by
adding a new leaf node to the tree.

In each step, DESPOT chooses the action that maximizes the upper bound
$\bar{V}[b]$ and the observation that maximizes \emph{weighted excess
uncertainty}, a regularized version of the gap $\bar{V}[b] - \underbar{V}[b]$
between the upper and lower bounds. This search strategy heuristically focuses
exploration on the optimally reachable belief
space~\cite{kurniawati2008sarsop}.  Finally, DESPOT backs up the upper and
lower bounds of all nodes visited by the trial.

\vspace{0.25em}

\noindent There are key two challenges in solving Lat-POMDP:

First, we must construct the configuration lattice. We use lazy evaluation to
interleave lattice construction with planning.  By doing so, we only evaluate
the parts of the lattice that are visited by DESPOT
(\cref{sec:planner-lattice}).

Second, we must provide DESPOT with upper and lower bounds on the optimal value
function to guide its search.  We derive these bounds from a relaxed version of
the problem that ignores obstacles by only considering the pose of the movable
object relative to the hand
(\cref{sec:planner-cartesian,sec:planner-upperbound,sec:planner-lowerbound}).

\subsection{Lattice Construction}
\label{sec:planner-lattice}
DESPOT uses upper and lower bounds to focus its search on belief states that
are likely to be visited by the optimal policy. We exploit this fact to avoid
constructing the entire lattice. Instead, we interleave lattice construction
with planning and only instantiate the lattice edges visited by the search,
similar to the concept \emph{lazy evaluation} used in motion
planning~\cite{bohlin2000path,hauser2015lazy}.

We begin with no pre-computation and run DESPOT until it queries the transition
model $\Tlattice$, observation model $\Omegalattice$, or reward function
$\Rlattice$ for a state-action pair $(\xr, \alattice)$ that has not yet been
evaluated. When this occurs, we pause the search and check the feasibility of
the action by running the local planner $\phi(\xr, \alattice)$. We use the
outcome of the local planner to update the Lat-POMDP model and resume the
search.

\Cref{fig:planner-lazy} shows the (\subref{fig:planner-lazy-full})~full lattice
and (\subref{fig:planner-lazy-partial})~subset evaluated by DESPOT. Feasible
action templates are shown as green, infeasible action templates are shown as
red, and unevaluated action templates are shown as gray. Note that DESPOT
evaluates only a small number of state-action pairs, obviating the need for
performing computationally expensive collision checks on much of the lattice.

Note that it is also possible to use a hybrid approach by evaluating some parts
of the lattice offline and deferring others to be computed online. For example,
we may compute inverse kinematics solutions, kinematic feasibility checks, and
self-collision checks in an offline pre-computation step.  These values are
fixed for a given support surface and, thus, can be used across multiple
problem instances. Other feasibility tests, e.g. collision with obstacles in
the environment, are lazily evaluated online.

\begin{figure}[t]%
  \centering%
  \begin{subfigure}[t]{0.49\columnwidth}%
    \includegraphics[width=\textwidth]{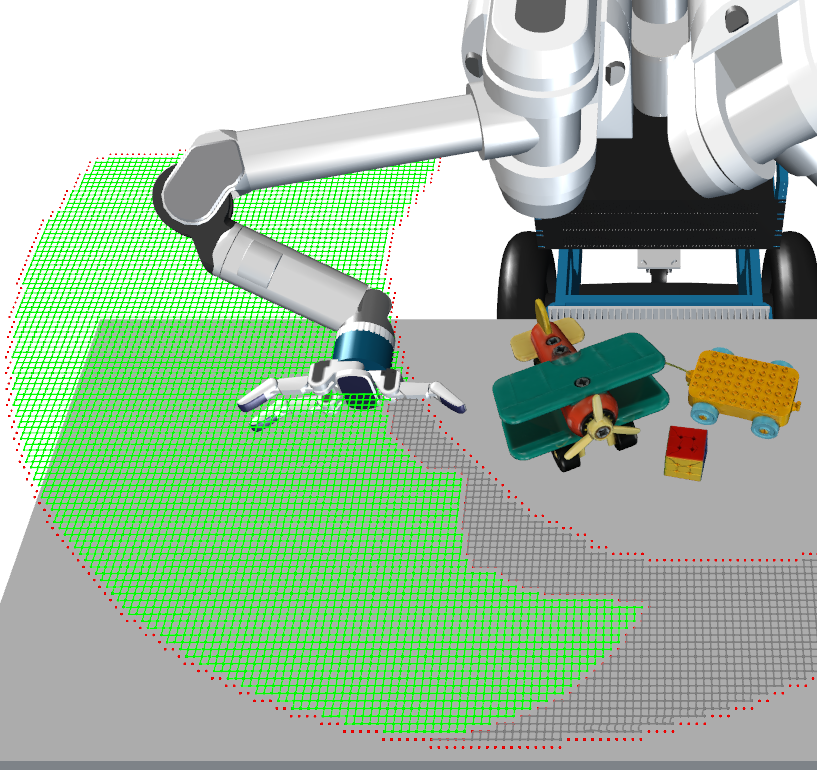}%
    \caption{Full Lattice}%
    \label{fig:planner-lazy-full}%
  \end{subfigure}
  \begin{subfigure}[t]{0.49\columnwidth}%
    \includegraphics[width=\textwidth]{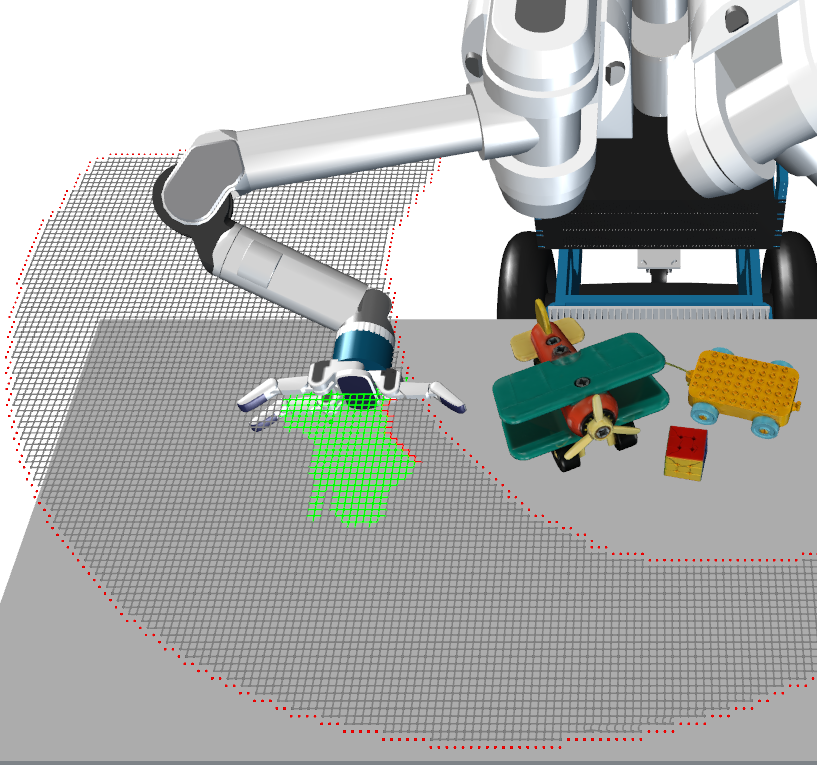}%
    \caption{Visited Lattice}%
    \label{fig:planner-lazy-partial}%
  \end{subfigure}%
  \caption[]{%
    Comparison of the (\subref{fig:planner-lazy-full})~full lattice to the
    (\subref{fig:planner-lazy-partial})~subset visited by DESPOT, with feasible
    (\feasibleswatch), infeasible (\infeasibleswatch) and unevaluated
    (\unevaluatedswatch) edges. We interleave planning with lattice
    construction to construct only the subset of the lattice visited by the
    search. This figure is best viewed in color.
  }%
  \label{fig:planner-lazy}%
\end{figure}

\subsection{Hand-Relative POMDP}
\label{sec:planner-cartesian}
Recall from \cref{sec:lattice} that the motion of the end effector and the
object is independent of the pose of the end effector $\xr$ or the robot
configuration $q$. We use this insight to define a \emph{hand-relative POMDP}
Rel-POMDP $(\Srel, \Alattice, O, \Trel, \Omegarel, \Rrel)$ with a state space
that only includes the pose $\xorel = x_\text{r}^{-1} \xo \in \Srel$ of the
movable object relative to the hand.  The hand-relative transition model
$\Trel$, observation model $\Omegarel$, and reward function $\Rrel$ are
identical to the original model when no un-modelled contact occurs. Rel-POMDP
is identical to the hand-relative POMDP models used in prior
work~\cite{horowitz2013interactive,koval2015manifold_ijrr,koval2015precontact_ijrr}
and is equivalent to assuming that environment is empty and the robot is a lone
end effector actuated by an incorporeal planar joint.

\subsection{Lat-POMDP Upper Bound}
\label{sec:planner-upperbound}
Rel-POMDP is a relaxation of Lat-POMDP that treats all actions as feasible. As
such:

\begin{theorem}
  The optimal value function $\Vrel^*$ of Rel-POMDP is an upper bound on the
  optimal value function $\Vlattice^*$ of Lat-POMDP:
  $\Vrel^*[b] \ge \Vlattice^*[b]$ for all $b \in \Delta$.
  \label{thm:planner-upperbound}
\end{theorem}
\begin{proof}
  We define a \emph{scenario} $\psi = (s_0, \psi_1, \psi_2, \dotsc)$ as an
  abstract simulation trajectory that captures all uncertainty in our POMDP
  model~\cite{ng2000pegasus,somani2013despot}. A scenario is generated by
  drawing the initial state $s_o \sim \bel(s_0)$ from the initial belief state
  and each random number $\psi_1 \sim \uniform[0, 1]$ uniformly from the unit
  interval. Given a scenario $\psi$, we assume that the outcome of executing a
  sequence of actions is deterministic; i.e. all stochasticity is captured in
  the initial state $s_0$ and the sequence of random numbers $\psi_1, \psi_2,
  \dotsc$.

  Suppose we have a policy $\pi$ for Rel-POMDP that executes the sequence of
  actions $a_\text{lat,1}, a_\text{lat,2}, \dotsc$ in scenario $\psi$. The
  policy visits the sequence of states $s_{\text{rel},1}, s_{\text{rel},2},
  \dotsc$ and receives the sequence of rewards $R_1, R_2, \dotsc$.

  Now consider executing $\pi$ in the same scenario $\psi$ on Lat-POMDP.
  Without loss of generality, assume that $\pi$ first takes an infeasible
  action or makes un-modelled contact with the environment at timestep $H$. The
  policy receives the same sequence of rewards $R_1, R_2, \dotsc, R_2, \dotsc,
  R_{H-1}, -1, -1, \dotsc$ as it did on Rel-POMDP until timestep $H$. Then, it
  receives $-1$ reward for taking an infeasible action, transitions to the
  absorbing state $\sinvalid$, and receives $-1$ reward for all time.

  Policy $\pi$ achieves value $V^\pi_{\text{rel},\psi} = \sum_{t=0}^\infty
  \gamma^t R_t$ on Rel-POMDP and $V^\pi_{\text{lat},\psi} = \sum_{t=0}^{H-1}
  \gamma^t R_t - \frac{\gamma^H}{1 - \gamma}$ on Lat-POMDP in scenario $\psi$.
  Since $R_t \ge -1$, we know that $V^\pi_{\text{rel},\psi} \ge
  V^\pi_{\text{lat},\psi}$. The value of a policy $V^\pi =
  E_{\psi}[V_{\psi}^\pi]$ is the expected value of $\pi$ over all scenarios.

  Consider the optimal policy $\pi_\text{lat}^*$ of Lat-POMDP. There exists
  some Rel-POMDP policy $\pi_\text{mimic}$ that executes the same sequence of
  actions as $\pi_\text{lat}^*$ in all scenarios. From the reasoning above, we
  know that $V^{\pi_\text{mimic}}_\text{rel} \ge V^*_\text{lat}$. We also know
  that  $V_\text{rel}^* \ge V^{\pi_\text{mimic}}_\text{rel}$ because the value
  of any policy is a lower bound on the optimal value function.

  Therefore, $V_\text{rel}^* \ge V^{\pi_\text{mimic}}_\text{rel} \ge
  V^*_\text{lat}$; i.e. the optimal value function $V_\text{rel}^*$ of
  Rel-POMDP is an upper bound on the optimal value function $V_\text{lat}^*$ of
  Lat-POMDP.
\end{proof}

This result implies that \emph{any} upper bound $\bar{V}_\text{rel}$ is an
upper bound on the value of the optimal value function $\bar{V}_\text{rel} \ge
\Vrel^* \ge \Vlattice^*$. Therefore, we may also use $\bar{V}_\text{rel}$ as an
upper bound on Lat-POMDP. The key advantage of doing so is that the
$\bar{V}_\text{rel}$ may be pre-computed once per hand-object pair. In
contrast, the same upper bound on $\bar{V}_\text{lat}$ must be re-computed for
each problem instance.

\subsection{Lower Bound}
\label{sec:planner-lowerbound}
We also use Rel-POMDP as a convenient method of constructing a lower bound
$\underbar{V}_\text{lat}$ for Lat-POMDP. As explained above, the value of any
policy is a lower bound on the optimal value function. We use offline
pre-computation to compute a rollout policy on $\pi_\text{rollout}$ for
Rel-POMDP once per hand-object pair. For example, we could use MDP value
iteration to compute a QMDP policy~\cite{littman1995learning} or a point-based
method~\cite{pineau2003point,kurniawati2008sarsop} to find a near-optimal
policy.

Given an arbitrary policy $\pi_\text{rollout}$ computed in this way, we
construct an approximate lower bound $\underbar{V}_\text{lat}$ for Lat-POMDP by
estimating the value $V_\text{lat}^{\pi_\text{rollout}}$ of executing
$\pi_\text{rollout}$ on Lat-POMDP using rollouts. Approximating a lower bound
with a \emph{rollout policy} is commonly used in POMCP~\cite{silver2010monte},
DESPOT~\cite{somani2013despot}, and other online POMDP solvers.

\section{Experimental Results}
\label{sec:simulation}
We validated the efficacy of the proposed algorithm by running simulation
experiments on a robot equipped with a 7-DOF Barrett WAM
arm~\cite{salisbury1988preliminary} and the
BarrettHand~\cite{townsend2000barretthand} end-effector. The robot attempts to
push a bottle into the center of its palm on a table littered with obstacles.

\subsection{Problem Definition}
The state space of the problem consists of the configuration space $Q =
\mathbb{R}^7$ of the robot and the pose of the object $\Xo$ relative to the end
effector. The robot begins in a known initial configuration $q_0$ and the
initial pose of the bottle $\xo$ is drawn from a Gaussian distribution centered
in front of the palm with a covariance matrix of
$\Sigma^{1/2} = \diag [ 5~\text{mm}, 10~\text{cm} ]$. At each timestep, the robot
chooses an action $\alattice$ that moves $1~\text{cm}$ at a constant Cartesian
velocity in the $xy$-plane, receives an observation $o$ from its $\no = 2$
fingertip contact sensors, and receives a reward $R$.

The goal is to push the object into the $4~\text{ cm} \times 6~\text{ cm}$ goal
region $\Xgoal$ centered in front of the palm. We evaluate the performance of
the policy by: (1) computing the probability that the object is in the goal
region at each timestep and (2) computing the discounted sum of reward with
$\gamma = 0.99$. If the robot takes an infeasible action, the simulation
immediately terminates and the robot receives $-1$ reward for all remaining
timesteps

\subsubsection{Transition Model}
The motion of the object is assumed to be
quasistatic~\cite{lynch1992manipulation} and is simulated using the Box2D
physics simulator~\cite{catto2010box2d}. We simulate uncertainty in the model
by sampling the hand-object friction coefficient and center of the object-table
pressure distribution from Gaussian distributions at each
timestep~\cite{duff2010motion,koval2015manifold_ijrr}.

\subsubsection{Observation Model}
After each timestep the robot receives a binary observation from a contact
sensor on each of its fingertips. We assume that the sensors perfectly
discriminate between contact and
no-contact~\cite{koval2015manifold_ijrr,koval2015precontact_ijrr}, but provide
no additional information about where contact occurred on the sensor. The robot
must take information-gathering actions, by moving side-to-side, to localize
the object relative to the hand.

\subsubsection{Rel-POMDP Discretization}
\label{sec:simulation-discretization}
We discretize $\Srel$ with a 1~cm resolution over a region of size
$20~\text{cm} \times 44~\text{cm}$ centered around the palm. We compute a
transition model, observation model, and reward function for Rel-POMDP by
taking the expectation over the continuous models by assuming that each
discrete state represents a uniform distribution over the corresponding area of
the continuous state space. We use these discrete models in both our Rel-POMDP
and Lat-POMDP experiments.

States that leave the discretized region around the hand are treated as
un-modelled contact. As in prior
work~\cite{horowitz2013interactive,koval2015precontact_ijrr}, discretization is
necessary to speed up evaluation of the model and to enable calculation of the
QMDP and SARSOP baseline policies described below.

\subsection{State Lattice Construction}
The robot's actions tessellate $\Xr$ into a lattice centered at
$\Tabs{\text{ee}}(q_0)$ with a resolution of $\Delta \xr = \Delta \yr =
1~\text{cm}$. First, we select a configuration $\qlattice(\xr)$ using an
iterative inverse kinematics solver initialized with the solution of an
adjacent lattice point. Then, we use a Cartesian motion planner to find a
trajectory that connects adjacent points while satisfying the constraints
imposed by the action templates. In most configurations, the 1~cm lattice is
sufficiently dense to simply connect adjacent lattice points with a
straight-line trajectory in configuration space. Forward kinematics, inverse
kinematics, and collision detection is provided by the Dynamic Animation and
Robotics Toolkit (DART)~\cite{unknown2015dynamic}.

As described in \cref{sec:planner-lattice}, the kinematic structure of the
lattice is computed offline for the height of the support surface, but no
collision checking is performed. Collision checks are deferred to runtime when
the planner queries the feasibility of an action template. At that point, the
edge is checked for collision against the environment using the Flexible
Collision Library (FCL)~\cite{pan2012fcl}.

\subsection{Policies}
We compare the quality of the policy produced by the proposed algorithm
(Lat-DESPOT) against several policies:

\subsubsection{Rel-QMDP}
Choose the action at each timestep that greedily optimizes the single-step
$Q$-value of the MDP value function associated from
Rel-POMDP~\cite{littman1995learning}. QMDP is optimal if all uncertainty
disappears after executing one action and, as a result, does not perform
multi-step information-gathering actions. QMDP is commonly used in robotic
applications due to its speed, simplicity, and good performance in domains
where information is easily
gathered~\cite{emery2004approximate,javdani2015shared}.

\subsubsection{Rel-SARSOP}
Compute a near-optimal policy for Rel-POMDP using SARSOP, an offline
point-based method~\cite{kurniawati2008sarsop}. Rel-SARSOP serves
as a baseline to demonstrate that information-gathering is beneficial and to
verify that DESPOT is capable of finding a near-optimal solution. SARSOP has
been shown to perform well on Rel-POMDP in prior
work~\cite{horowitz2013interactive,koval2015precontact_ijrr}. As in that work,
we used the implementation of SARSOP provided by the APPL toolkit and allowed it
to run for 10~minutes.

\subsubsection{Rel-DESPOT}
Use DESPOT to find a near-optimal policy for Rel-POMDP~\cite{somani2013despot}.
The solver uses Rel-QMDP as an upper bound and rollouts of Rel-QMDP as a lower
bound (see \cref{sec:planner-lowerbound} for an explanation of a rollout
policy). We use the implementation of DESPOT provided by the APPL toolkit and
tuned the number of trials, number of scenarios, regularization constant, and
gap constant on a set of training problem instances that are distinct from the
results presented in this section.

\subsubsection{Lift-QMDP and Lift-SARSOP}
Use the state lattice to evaluate the feasibility of the action returned by
Rel-QMDP and Rel-SARSOP, respectively, before executing it. If the desired
action is infeasible, instead execute the action with the highest estimated
$Q$-value. This represents a heuristic solution for modifying a Rel-POMDP
policy to avoid taking infeasible actions.

\subsubsection{Lat-DESPOT}
The proposed algorithm described in \cref{sec:planner}. We run
DESPOT~\cite{somani2013despot} on Lat-POMDP using Rel-QMDP as the upper bound
and rollouts of Lift-QMDP as the lower bound (see \cref{sec:planner-lowerbound}
for an explanation of a rollout policy). Just as with Rel-DESPOT, we used the
implementation of DESPOT provided by the APPL toolkit and tuned all parameters
on a set of training problem instances.

\subsection{Rel-POMDP Experiments}
\label{sec:simulation-rel}
\begin{figure}[t]%
  \centering%
  \includegraphics{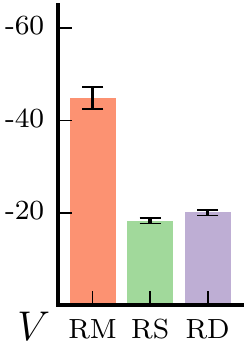}%
  \includegraphics{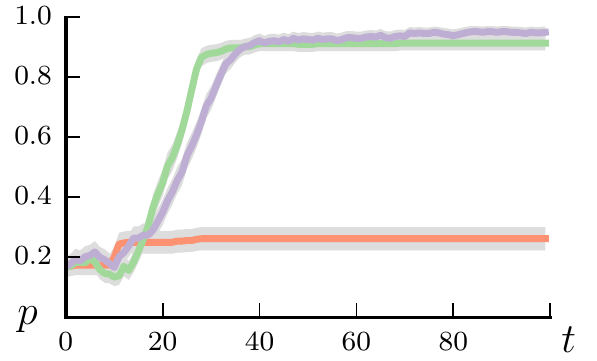}%
  \caption[]{%
    Performance of the 
    Rel-QMDP (RM \mdpswatch), 
    Rel-SARSOP (RS \sarsopswatch), and
    Rel-DESPOT (RD \despotswatch) policies on Rel-POMDP.
    (Left)~Value $V_\text{rel}$ achieved by each policy computed over 100
    timesteps. Note that the $y$-axis is inverted; lower (less negative) is
    better.
    (Right)~Probability that the movable object is in $\Xgoal$ as a function of
    timestep.
    Error bars denote a 95\% confidence interval. Best viewed in color.
  }%
  \label{fig:simulation-rel}%
\end{figure}

We begin with a preliminary experiment on Rel-POMDP. Our goal is to: (1)
validate that discretizing Rel-POMDP does not harm the performance of the
optimal policy on the underlying continuous problem, (2) confirm that
information-gathering is necessary, (3) verify that Rel-DESPOT does not perform
worse than Rel-SARSOP, and (4) estimate $\Vrel^*[b_0]$.

\subsubsection{Discretization Validation}
\Cref{fig:simulation-rel} shows simulation results averaged over
$500$~instances of the experiment described above.
\Cref{fig:simulation-rel}-Left shows the value of each policy achieved in
100~timesteps on simulated on the discretized Rel-POMDP problem described in
\cref{sec:simulation-discretization}. \Cref{fig:simulation-rel}-Right shows the
probability that the movable object is in $\Xgoal$ at each timestep when
simulated using the continuous model. The close agreement between the results
suggests that discretizing the state space does not harm the performance of the
policy on the underlying continuous problem.

\subsubsection{Necessity of Information-Gathering}
Rel-QMDP (\mdpswatch) performs poorly on this problem, achieving $<30\%$
success probability, because it pushes straight and does not attempt to
localize the object. Rel-SARSOP (\sarsopswatch) and Rel-DESPOT (\despotswatch)
execute information-gathering action by moving the hand laterally to drive the
movable object into one of the fingertip contact sensors. Once the object has
been localized, the policy successfully pushes it into the goal region. These
results are consistent with prior
work~\cite{horowitz2013interactive,koval2015precontact_ijrr} and confirm that
information-gathering is necessary to perform well on this problem.

\subsubsection{Rel-DESPOT Policy}
\label{sec:simulation-h2}
Our intuition is that it is more difficult to solve Lat-POMDP than Rel-POMDP.
Therefore, it is important that we verify that DESPOT can successfully solve
Rel-POMDP before applying it to Lat-POMDP. Our results confirm this is true:
Rel-DESPOT (\despotswatch) achieves comparable value and success probability to
Rel-SARSOP (\sarsopswatch).

\subsection{Lat-POMDP Experiments}
\begin{figure*}[t]%
  \centering%
  \begin{subfigure}[t]{0.24\textwidth}%
    \centering%
    \includegraphics[width=\textwidth]{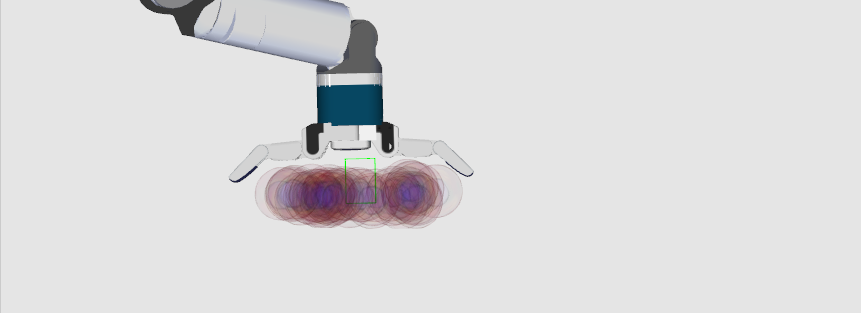} \\[1ex]
    \raggedleft%
    \includegraphics{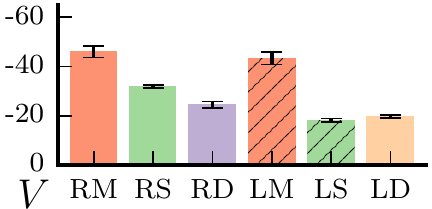} \\[1ex]
    \includegraphics{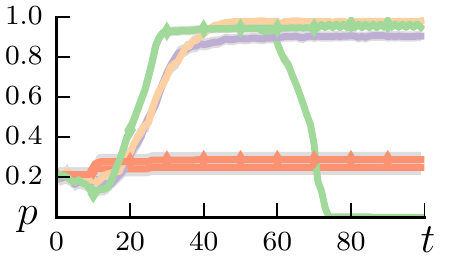} \\
    \includegraphics{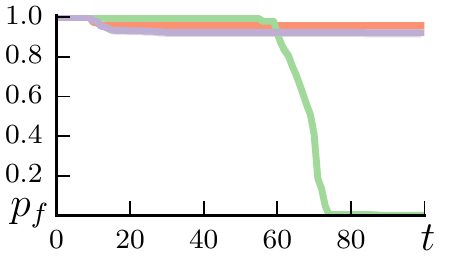} \\
    \caption{Empty Table}%
    \label{fig:simulation-lat-empty}%
  \end{subfigure}%
  \hspace{0.01\textwidth}%
  \begin{subfigure}[t]{0.24\textwidth}%
    \centering%
    \includegraphics[width=\textwidth]{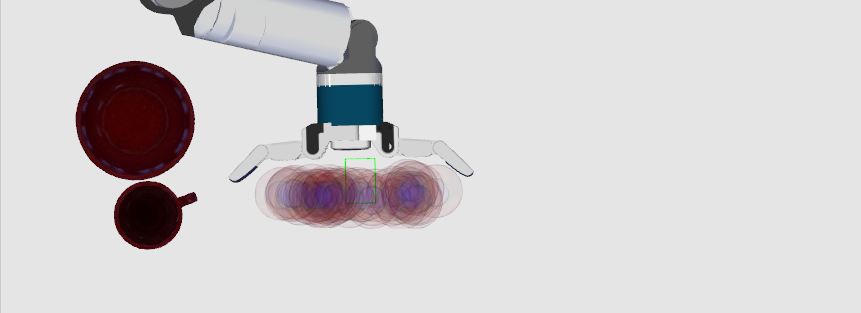} \\[1ex]
    \raggedleft%
    \includegraphics{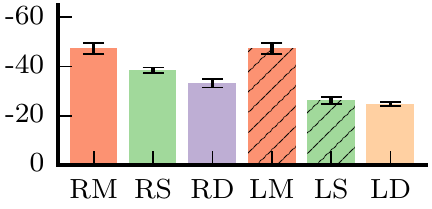} \\
    \vspace{0.95ex}%
    \includegraphics{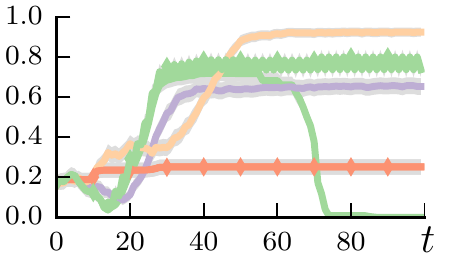} \\
    \includegraphics{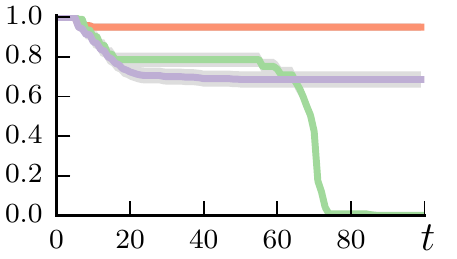} \\
    \caption{Right Obstacles}%
    \label{fig:simulation-lat-right}%
  \end{subfigure}%
  \hspace{0.01\textwidth}%
  \begin{subfigure}[t]{0.24\textwidth}%
    \centering%
    \includegraphics[width=\textwidth]{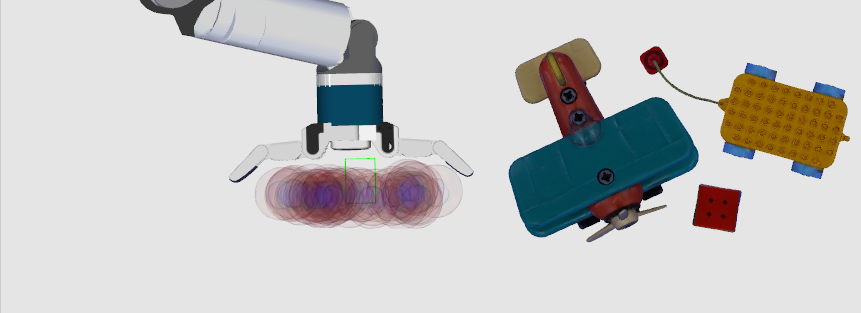} \\[1ex]
    \raggedleft%
    \includegraphics{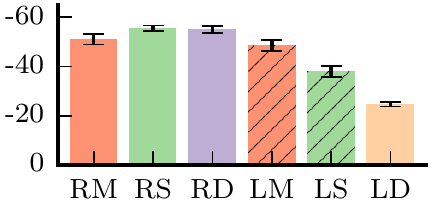} \\
    \vspace{0.95ex}%
    \includegraphics{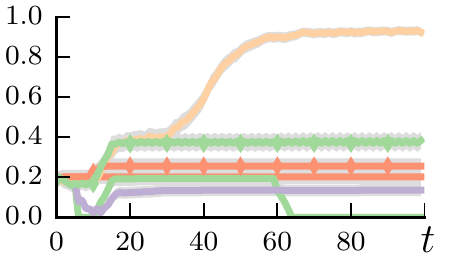} \\
    \includegraphics{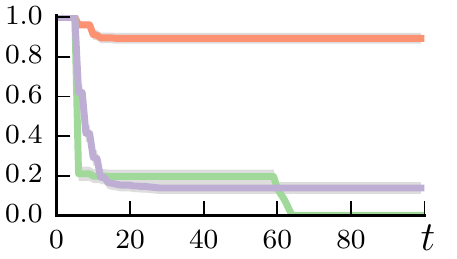} \\
    \caption{Left Obstacles}%
    \label{fig:simulation-lat-left}%
  \end{subfigure}%
  \hspace{0.01\textwidth}%
  \begin{subfigure}[t]{0.24\textwidth}%
    \centering%
    \includegraphics[width=\textwidth]{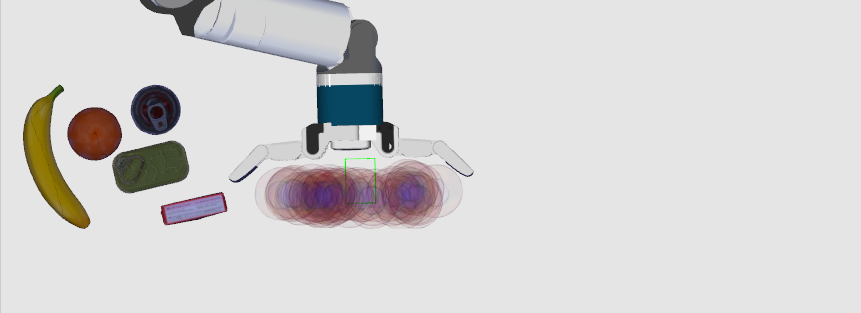} \\[1ex]
    \raggedleft%
    \includegraphics{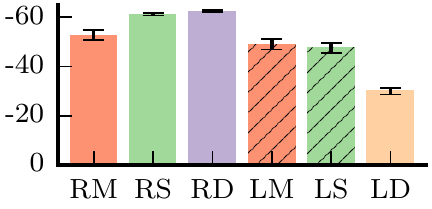} \\
    \vspace{0.95ex}%
    \includegraphics{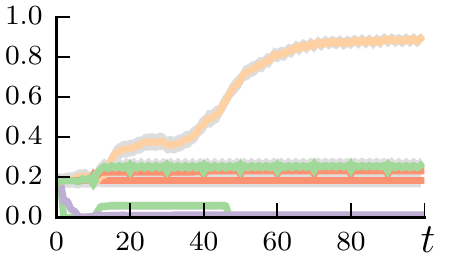} \\
    \includegraphics{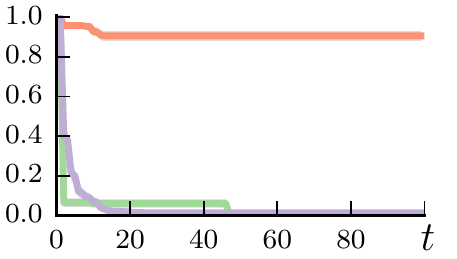} \\
    \caption{Complex Obstacles}%
    \label{fig:simulation-lat-complex}%
  \end{subfigure}%
  \caption[]{%
    Performance of the Rel-QMDP (RM \mdpswatch), Rel-SARSOP (RS \sarsopswatch),
    Rel-DESPOT (RD \despotswatch), Lift-QMDP (LM \liftmdpswatch), Lift-SARSOP
    (LS \liftsarsopswatch), and Lat-DESPOT (LD \latdespotswatch) policies on
    Lat-POMDP on four environments.
    (Top)~Value achieved by each policy after 100 timesteps. Note that the
    $y$-axis is inverted; lower (less negative) is better.
    (Middle)~Probability $p = \Prob(s_t \in \Xgoal)$ that the movable object is
    in the goal region at each timestep.
    (Bottom)~Probability that the execution is feasible as a function of time.
    Lift-QMDP, Lift-SARSOP, and Lat-DESPOT are omitted because they do not take
    infeasible actions.
    In all cases, error bars denote a 95\% confidence interval.
    Best viewed in color.
  }%
  \label{fig:simulation-lat}%
\end{figure*}

We evaluate the proposed approach (Lat-DESPOT) on Lat-POMDP in four different
environments: (a)~an empty table, (b)~obstacles on the right, (c)~obstacles on
the left, and (d)~more complex obstacles on the right. Note that kinematic
constraints are present in all four environments, even the empty table, in the
form of reachability limits, self-collision, and collision between the arm and
the table. Scenes (b), (c), and (d) are constructed out of objects selected
from the YCB dataset~\cite{calli2015ycb}.

\Cref{fig:simulation-lat} shows results for each scene averaged over
$500$~instances of the problem. Just as in the Rel-POMDP experiments,
\Cref{fig:simulation-lat}-Top shows the value $\Vlattice$ achieved by each
policy when evaluated on the discretized version of Lat-POMDP.
\Cref{fig:simulation-lat}-Middle shows the probability that the movable object
is in $\Xgoal$ at each timestep, treating instances that have terminated as
zero probability. \Cref{fig:simulation-lat}-Bottom shows the proportion of
Rel-QMDP, Rel-SARSOP, and Rel-DESPOT policies that are active at each timestep;
i.e. have not yet terminated by taking an infeasible action. Note that all
six policies---even Rel-MDP, Rel-SARSOP, and Rel-DESPOT---are subject to
kinematic constraints during execution.

\subsubsection{Baseline Policies}
Rel-QMDP (RM \mdpswatch) and Lift-QMDP (LM \liftmdpswatch) to perform poorly
across all environments, achieving $<30\%$ success probability, because they do
not take multi-step information-gathering actions.  \Cref{fig:simulation-lat}
confirms this: both QMDP policies perform poorly on all four environments.

Rel-SARSOP (RS \sarsopswatch) and Rel-DESPOT (RD \despotswatch) perform well on
environments (a) and (b) because they hit obstacles late in execution. The
converse is true on environments (c) and (d): both policies hit obstacles so
quickly that they perform worse than Rel-QMDP!

Lift-SARSOP (LS \liftsarsopswatch) performs near-optimally on environments (a)
and (b) because it: (1) does not take infeasible actions and (2) gathers
information. However, it performs no better than Rel-QMDP on problem (d). This
occurs because Lift-SARSOP myopically considers obstacles in a one-step
lookahead and may oscillate when blocked. Small changes in the environment are
sufficient to induce this behavior: the key difference between environments (b)
and (d) is the introduction of a red box that creates a cul-de-sac in the
lattice.

\subsubsection{Lat-DESPOT Policy}
Our proposed approach, Lat-DESPOT (\latdespotswatch), avoids myopic behavior by
considering action feasibility during planning.  Lat-DESPOT performs no worse
than Lift-SARSOP on environments (a) and (b) and outperforms it on environments
(c) and (d). Unlike Rel-SARSOP, Lat-DESPOT identifies the cul-de-sac in (d)
during planning and avoids becoming trapped in it. In summary, \emph{Lat-DESPOT
is the only policy that performs near-optimally on all four environments.}

Our unoptimized implementation of Lat-DESPOT took between 200~$\mu\text{s}$ and
2.4~s to select an action (a 1~cm motion of the end-effector) on a single core
of a 4 GHz Intel Core i7 CPU. The policy was slowest to evaluate early in
execution, when information-gathering is necessary, and fastest once the
movable object is localized; i.e. once the upper and lower bounds become tight.
The QMDP and SARSOP policies, which perform no planning online, took an average
of 1.6~$\mu\text{s}$ and 218~$\mu\text{s}$ respectively. We are optimistic
about achieving real-time performance from Lat-DESPOT by parallelizing scenario
rollouts, optimizing our implementation of the algorithm, and leveraging the
loose coupling between the optimal pre- and post-contact
policies~\cite{koval2015precontact_ijrr}.

\subsubsection{Upper Bound Validation}
Finally, we combine the data in \cref{fig:simulation-rel}-Left and
\cref{fig:simulation-lat}-Top to empirically verify the bound we proved in
\cref{thm:planner-upperbound}; i.e. the optimal value function $\Vrel^*$ of
Rel-POMDP is an upper bound on the optimal value function $\Vlattice^*$ of
Lat-POMDP. Note that the value of Rel-SARSOP (\sarsopswatch) and Rel-DESPOT
(\despotswatch) on Rel-POMDP (\cref{fig:simulation-rel}-Left) are greater (i.e.
less negative) than the value of all policies we evaluated on Lat-POMDP
(\cref{fig:simulation-lat}-Top). The data supports our theory: the optimal
value achieved on Rel-POMDP is no worse than the highest value achieved on
Lat-POMDP in environment (a) and greater than the highest value achieved in
environments (b), (c), and (d).

\section{Discussion}
\label{sec:discussion}
In this paper, we formulated the problem of planar contact manipulation under
uncertainty as a POMDP in the joint space of robot configurations and poses of
the movable object (\cref{sec:intro}). We simplify the problem by constructing
a lattice in the robot's configuration space and prove that, under mild
assumptions, the optimal policy of Lat-POMDP will never take an infeasible
action (\cref{sec:lattice}). We find a near-optimal policy for Lat-POMDP using
DESPOT~\cite{somani2013despot}, a state-of-the-art online POMDP solver, guided by
upper and lower bounds derived from Rel-POMDP (\cref{sec:planner}).

Our simulation results show that Lat-DESPOT outperforms five baseline
algorithms on cluttered environments: it achieves a $>90\%$ success rate on all
environments, compared to the best baseline (Lift-SARSOP) that achieves only a
$~20\%$ success rate on difficult problems. They also highlight the importance
of reasoning about both object pose uncertainty and kinematic constraints
during planning. Lat-DESPOT is a promising first step towards using recent
advances in online POMDP solvers, like DESPOT~\cite{somani2013despot}, to
achieve that goal. However, Lat-DESPOT has several limitations that we plan to
address in future work:

First, our approach assumes that the robot has perfect proprioception and
operates in an environment with known obstacles. In practice, robots often
have imperfect proprioception~\cite{klingensmith2013closed,boots2014learning}
and uncertainty about the pose of \emph{all} objects in the environment. We
hope to relax both of these assumptions by replacing the deterministic
transition model for robot configuration with a stochastic model that considers
the probability of hitting an obstacle.

Second, we are excited to scale our approach up a larger repertoire of action
templates (including non-planar motion), solving more complex tasks, and
planning in environments that contain multiple movable objects. Solving these
more complex problems will require more informative heuristics. We are
optimistic that more sophisticated Rel-POMDP policies, e.g. those computed by a
point-based method~\cite{porta2006point} or Monte Carlo Value
Iteration~\cite{bai2011monte}, could be used to guide the search.

Third, our approach commits to a single inverse kinematics solution
$\qlattice(\xr)$ for each lattice point. This prevents robots like HERB, which
has a seven degree-of-freedom manipulator, from using redundancy to
avoid kinematic constraints. We plan to relax this assumption in future
work by generating multiple inverse kinematic solutions for each lattice point
and instantiating an action template for each. Our intuition is that many
solutions share the same connectivity and, thus, may be treated identically
during planning.

Finally, we plan to implement Lat-DESPOT on a real robotic manipulator and
evaluate the performance of our approach on real-world manipulation tasks.

\section*{Acknowledgements}
\noindent This work was supported by a NASA Space Technology Research
Fellowship (award NNX13AL62H), the National Science Foundation (awards
IIS-1218182 and IIS-1409003), the U.S. Office of Naval Research, and the Toyota
Motor Corporation. We would like to thank Rachel Holladay, Shervin Javdani,
Jennifer King, Stefanos Nikolaidis, and the members of the Personal Robotics
Lab for their helpful input. We would also like to thank Nan Ye for assistance
with the APPL toolkit.

\vfill\break

\bibliographystyle{plainnat}
\bibliography{references}

\end{document}

%% file: figures/formulation_state.pdf_tex
%% Creator: Inkscape inkscape 0.48.4, www.inkscape.org
%% Accompanies image file 'formulation_state.pdf' (pdf, eps, ps)
\begingroup%
  \makeatletter%
  \providecommand\color[2][]{%
    \errmessage{(Inkscape) Color is used for the text in Inkscape, but the package 'color.sty' is not loaded}%
    \renewcommand\color[2][]{}%
  }%
  \providecommand\transparent[1]{%
    \errmessage{(Inkscape) Transparency is used (non-zero) for the text in Inkscape, but the package 'transparent.sty' is not loaded}%
    \renewcommand\transparent[1]{}%
  }%
  \providecommand\rotatebox[2]{#2}%
  \ifx\svgwidth\undefined%
    \setlength{\unitlength}{108bp}%
    \ifx\svgscale\undefined%
      \relax%
    \else%
      \setlength{\unitlength}{\unitlength * \real{\svgscale}}%
    \fi%
  \else%
    \setlength{\unitlength}{\svgwidth}%
  \fi%
  \global\let\svgwidth\undefined%
  \global\let\svgscale\undefined%
  \makeatother%
  \begin{picture}(1,1.37907557)%
    \put(0,0){\includegraphics[width=\unitlength]{formulation_state.pdf}}%
    \put(0.57936509,0.07539667){\color[rgb]{0,0,0}\makebox(0,0)[lb]{\smash{$x_r$}}}%
    \put(0.03010157,0.73801874){\color[rgb]{0,0,0}\makebox(0,0)[lb]{\smash{$q$}}}%
  \end{picture}%
\endgroup%

%% file: paper.bbl
\begin{thebibliography}{52}
\providecommand{\natexlab}[1]{#1}
\providecommand{\url}[1]{\texttt{#1}}
\expandafter\ifx\csname urlstyle\endcsname\relax
  \providecommand{\doi}[1]{doi: #1}\else
  \providecommand{\doi}{doi: \begingroup \urlstyle{rm}\Url}\fi

\bibitem[unk(2013)]{unknown2015dynamic}
{D}ynamic {A}nimation and {R}obotics {T}oolkit.
\newblock \url{http://dartsim.github.io}, 2013.

\bibitem[Bai et~al.(2011)Bai, Hsu, Lee, and Ngo]{bai2011monte}
H.~Bai, D.~Hsu, W.S. Lee, and V.A. Ngo.
\newblock {M}onte {C}arlo value iteration for continuous-state {POMDP}s.
\newblock In \emph{Workshop on the Algorithmic Foundations of Robotics}, 2011.

\bibitem[Bohlin and Kavraki(2000)]{bohlin2000path}
R.~Bohlin and L.E. Kavraki.
\newblock Path planning using lazy {PRM}.
\newblock In \emph{{IEEE} International Conference on Robotics and Automation},
  pages 521--528, 2000.

\bibitem[Boots et~al.(2014)Boots, Byravan, and Fox]{boots2014learning}
B.~Boots, A.~Byravan, and D.~Fox.
\newblock Learning predictive models of a depth camera \& manipulator from raw
  execution traces.
\newblock In \emph{{IEEE} International Conference on Robotics and Automation},
  2014.

\bibitem[Brokowski et~al.(1993)Brokowski, Peshkin, and
  Goldberg]{brokowski93curvedfence}
M.~Brokowski, M.~Peshkin, and K.~Goldberg.
\newblock Curved fences for part alignment.
\newblock In \emph{{IEEE} International Conference on Robotics and Automation},
  1993.
\newblock \doi{10.1109/ROBOT.1993.292216}.

\bibitem[Calli et~al.(2015)Calli, Singh, Walsman, Srinivasa, Abbeel, and
  Dollar]{calli2015ycb}
B.~Calli, A.~Singh, A.~Walsman, S.S. Srinivasa, P.~Abbeel, and A.M. Dollar.
\newblock The {YCB} object and model set: Towards common benchmarks for
  manipulation research.
\newblock In \emph{International Conference on Advanced Robotics}, 2015.

\bibitem[Catto(2010)]{catto2010box2d}
E.~Catto.
\newblock {Box2D}.
\newblock \url{http://box2d.org}, 2010.

\bibitem[Cohen et~al.(2013)Cohen, Chitta, and Likhachev]{cohen2013single}
B.~Cohen, S.~Chitta, and M.~Likhachev.
\newblock Single-and dual-arm motion planning with heuristic search.
\newblock \emph{International Journal of Robotics Research}, 2013.

\bibitem[Dogar et~al.(2012)Dogar, Hsiao, Ciocarlie, and
  Srinivasa]{dogar2012physics}
M.~Dogar, K.~Hsiao, M.~Ciocarlie, and S.S. Srinivasa.
\newblock Physics-based grasp planning through clutter.
\newblock In \emph{Robotics: Science and Systems}, 2012.

\bibitem[Dogar and Srinivasa(2010)]{dogar2010pushgrasp}
M.R. Dogar and S.S. Srinivasa.
\newblock Push-grasping with dexterous hands: Mechanics and a method.
\newblock In \emph{{IEEE/RSJ} International Conference on Intelligent Robots
  and Systems}, 2010.
\newblock \doi{10.1109/IROS.2010.5652970}.

\bibitem[Dogar and Srinivasa(2012)]{dogar2012planning}
M.R. Dogar and S.S. Srinivasa.
\newblock A planning framework for non-prehensile manipulation under clutter
  and uncertainty.
\newblock \emph{Autonomous Robots}, 33\penalty0 (3):\penalty0 217--236, 2012.
\newblock \doi{10.1007/s10514-012-9306-z}.

\bibitem[Duff et~al.(2010)Duff, Wyatt, and Stolkin]{duff2010motion}
D.J. Duff, J.~Wyatt, and R.~Stolkin.
\newblock Motion estimation using physical simulation.
\newblock In \emph{{IEEE} International Conference on Robotics and Automation},
  2010.
\newblock \doi{10.1109/ROBOT.2010.5509590}.

\bibitem[Emery-Montemerlo et~al.(2004)Emery-Montemerlo, Gordon, Schneider, and
  Thrun]{emery2004approximate}
R.~Emery-Montemerlo, G.~Gordon, J.~Schneider, and S.~Thrun.
\newblock Approximate solutions for partially observable stochastic games with
  common payoffs.
\newblock In \emph{Autonomous Agents and Multiagent Systems}, 2004.

\bibitem[Erdmann and Mason(1988)]{erdmann1988exploration}
M.A. Erdmann and M.T. Mason.
\newblock An exploration of sensorless manipulation.
\newblock \emph{{IEEE} Journal of Robotics and Automation}, 1988.
\newblock \doi{10.1109/56.800}.

\bibitem[Hauser(2015)]{hauser2015lazy}
K.~Hauser.
\newblock Lazy collision checking in asymptotically-optimal motion planning.
\newblock In \emph{{IEEE} International Conference on Robotics and Automation},
  2015.

\bibitem[Hebert et~al.(2013)Hebert, Howard, Hudson, Ma, and
  Burdick]{hebert2013next}
P.~Hebert, T.~Howard, N.~Hudson, J.~Ma, and J.W. Burdick.
\newblock The next best touch for model-based localization.
\newblock In \emph{{IEEE} International Conference on Robotics and Automation},
  2013.
\newblock \doi{10.1109/ICRA.2013.6630562}.

\bibitem[Horowitz and Burdick(2013)]{horowitz2013interactive}
M.~Horowitz and J.~Burdick.
\newblock Interactive non-prehensile manipulation for grasping via {POMDP}s.
\newblock In \emph{{IEEE} International Conference on Robotics and Automation},
  2013.
\newblock \doi{10.1109/ICRA.2013.6631031}.

\bibitem[Hsiao(2009)]{hsiao2009relatively}
K.~Hsiao.
\newblock \emph{Relatively robust grasping}.
\newblock PhD thesis, Massachusetts Institute of Technology, 2009.

\bibitem[Hsiao et~al.(2007)Hsiao, Kaelbling, and
  Lozano-P\`{e}rez]{hsiao2007grasping}
K.~Hsiao, L.P. Kaelbling, and T.~Lozano-P\`{e}rez.
\newblock Grasping {POMDP}s.
\newblock In \emph{{IEEE} International Conference on Robotics and Automation},
  2007.
\newblock \doi{10.1109/ROBOT.2007.364201}.

\bibitem[Hsiao et~al.(2008)Hsiao, Lozano-P\'{e}rez, and
  Kaelbling]{hsiao2008robust}
K.~Hsiao, T.~Lozano-P\'{e}rez, and L.P. Kaelbling.
\newblock Robust belief-based execution of manipulation programs.
\newblock In \emph{Workshop on the Algorithmic Foundations of Robotics}, 2008.

\bibitem[Javdani et~al.(2013)Javdani, Klingensmith, Bagnell, Pollard, and
  Srinivasa]{javdani2013efficient}
S.~Javdani, M.~Klingensmith, J.A. Bagnell, N.S. Pollard, and S.S. Srinivasa.
\newblock Efficient touch based localization through submodularity.
\newblock In \emph{{IEEE} International Conference on Robotics and Automation},
  2013.
\newblock \doi{10.1109/ICRA.2013.6630818}.

\bibitem[Javdani et~al.(2015)Javdani, Bagnell, and
  Srinivasa]{javdani2015shared}
S.~Javdani, J.A. Bagnell, and S.S. Srinivasa.
\newblock Shared autonomy via hindsight optimization.
\newblock In \emph{Robotics: Science and Systems}, 2015.

\bibitem[Kaelbling et~al.(1998)Kaelbling, Littman, and
  Cassandra]{kaelbling1998planning}
L.P. Kaelbling, M.L. Littman, and A.R. Cassandra.
\newblock Planning and acting in partially observable stochastic domains.
\newblock \emph{Artificial Intelligence}, 1998.
\newblock \doi{10.1016/S0004-3702(98)00023-X}.

\bibitem[Klingensmith et~al.(2013)Klingensmith, Galluzzo, Dellin, Kazemi,
  Bagnell, and Pollard]{klingensmith2013closed}
M.~Klingensmith, T.~Galluzzo, C.~Dellin, M.~Kazemi, J.A. Bagnell, and
  N.~Pollard.
\newblock Closed-loop servoing using real-time markerless arm tracking.
\newblock In \emph{{IEEE} International Conference on Robotics and
  AutomationHumanoids Workshop}, 2013.

\bibitem[Koval et~al.(2015{\natexlab{a}})Koval, King, Pollard, and
  Srinivasa]{koval2015robust}
M.C. Koval, J.E. King, N.S. Pollard, and S.S. Srinivasa.
\newblock Robust trajectory selection for rearrangement planning as a
  multi-armed bandit problem.
\newblock In \emph{{IEEE/RSJ} International Conference on Intelligent Robots
  and Systems}, 2015{\natexlab{a}}.

\bibitem[Koval et~al.(2015{\natexlab{b}})Koval, Pollard, and
  Srinivasa]{koval2015manifold_ijrr}
M.C. Koval, N.S. Pollard, and S.S. Srinivasa.
\newblock Pose estimation for planar contact manipulation with manifold
  particle filters.
\newblock \emph{International Journal of Robotics Research}, 34\penalty0
  (7):\penalty0 922--945, 2015{\natexlab{b}}.
\newblock \doi{10.1177/0278364915571007}.

\bibitem[Koval et~al.(2015{\natexlab{c}})Koval, Pollard, and
  Srinivasa]{koval2015precontact_ijrr}
M.C. Koval, N.S. Pollard, and S.S. Srinivasa.
\newblock Pre- and post-contact policy decomposition for planar contact
  manipulation under uncertainty.
\newblock \emph{International Journal of Robotics Research},
  2015{\natexlab{c}}.
\newblock \doi{10.1177/0278364915594474}.
\newblock In press.

\bibitem[Kurniawati et~al.(2008)Kurniawati, Hsu, and Lee]{kurniawati2008sarsop}
H.~Kurniawati, D.~Hsu, and W.S. Lee.
\newblock {SARSOP}: Efficient point-based {POMDP} planning by approximating
  optimally reachable belief spaces.
\newblock In \emph{Robotics: Science and Systems}, 2008.

\bibitem[LaValle and Hutchinson(1998)]{lavalle1998objective}
S.M. LaValle and S.A. Hutchinson.
\newblock An objective-based framework for motion planning under sensing and
  control uncertainties.
\newblock \emph{International Journal of Robotics Research}, 1998.
\newblock \doi{10.1177/027836499801700104}.

\bibitem[Li et~al.(2013)Li, Sch{\"u}rmann, Haschke, and Ritter]{li2013control}
Q.~Li, C.~Sch{\"u}rmann, R.~Haschke, and H.~Ritter.
\newblock A control framework for tactile servoing.
\newblock In \emph{Robotics: Science and Systems}, 2013.

\bibitem[Likhachev and Ferguson(2009)]{likhachev2009planning}
M.~Likhachev and D.~Ferguson.
\newblock Planning long dynamically feasible maneuvers for autonomous vehicles.
\newblock \emph{International Journal of Robotics Research}, 28\penalty0
  (8):\penalty0 933--945, 2009.

\bibitem[Littman et~al.(1995)Littman, Cassandra, and
  Kaelbling]{littman1995learning}
M.L. Littman, A.R. Cassandra, and L.P. Kaelbling.
\newblock Learning policies for partially observable environments: Scaling up.
\newblock \emph{International Conference on Machine Learning}, 1995.

\bibitem[Lynch et~al.(1992)Lynch, Maekawa, and Tanie]{lynch1992manipulation}
K.M. Lynch, H.~Maekawa, and K.~Tanie.
\newblock Manipulation and active sensing by pushing using tactile feedback.
\newblock In \emph{{IEEE/RSJ} International Conference on Intelligent Robots
  and Systems}, 1992.
\newblock \doi{10.1109/IROS.1992.587370}.

\bibitem[Mason(1986)]{mason1986mechanics}
M.T. Mason.
\newblock Mechanics and planning of manipulator pushing operations.
\newblock \emph{International Journal of Robotics Research}, 5\penalty0
  (3):\penalty0 53--71, 1986.
\newblock \doi{10.1177/027836498600500303}.

\bibitem[Ng and Jordan(2000)]{ng2000pegasus}
A.Y. Ng and M.~Jordan.
\newblock {PEGASUS}: A policy search method for large {MDP}s and {POMDP}s.
\newblock In \emph{Conference on Uncertainty in Artificial Intelligence}, 2000.

\bibitem[Ong et~al.(2009)Ong, Png, Hsu, and Lee]{ong2009pomdps}
S.C.W. Ong, S.W. Png, D.~Hsu, and W.S. Lee.
\newblock {POMDP}s for robotic tasks with mixed observability.
\newblock In \emph{Robotics: Science and Systems}, June 2009.

\bibitem[Pan et~al.(2012)Pan, Chitta, and Manocha]{pan2012fcl}
Jia Pan, Sachin Chitta, and Dinesh Manocha.
\newblock {FCL}: A general purpose library for collision and proximity queries.
\newblock In \emph{{IEEE} International Conference on Robotics and Automation},
  pages 3859--3866, 2012.

\bibitem[Pastor et~al.(2011)Pastor, Righetti, Kalakrishnan, and
  Schaal]{pastor2011online}
P.~Pastor, L.~Righetti, M.~Kalakrishnan, and S.~Schaal.
\newblock Online movement adaptation based on previous sensor experiences.
\newblock In \emph{{IEEE/RSJ} International Conference on Intelligent Robots
  and Systems}, 2011.
\newblock \doi{10.1109/IROS.2011.6095059}.

\bibitem[Petrovskaya and Khatib(2011)]{petrovskaya2011global}
A.~Petrovskaya and O.~Khatib.
\newblock Global localization of objects via touch.
\newblock \emph{{IEEE} Transactions on Robotics}, 27\penalty0 (3):\penalty0
  569--585, 2011.
\newblock \doi{10.1109/TRO.2011.2138450}.

\bibitem[Pineau et~al.(2003)Pineau, Gordon, and Thrun]{pineau2003point}
J.~Pineau, G.~Gordon, and S.~Thrun.
\newblock Point-based value iteration: An anytime algorithm for {POMDP}s.
\newblock In \emph{International Joint Conference on Artificial Intelligence},
  2003.

\bibitem[Pivtoraiko and Kelly(2005)]{pivtoraiko2005efficient}
M.~Pivtoraiko and A.~Kelly.
\newblock Efficient constrained path planning via search in state lattices.
\newblock In \emph{International Symposium on Artificial Intelligence,
  Robotics, and Automation in Space}, 2005.

\bibitem[Platt et~al.(2010)Platt, Fagg, and Grupen]{platt2010null}
R.~Platt, A.H. Fagg, and R.A. Grupen.
\newblock Nullspace grasp control: theory and experiments.
\newblock \emph{{IEEE} Transactions on Robotics}, 26\penalty0 (2):\penalty0
  282--295, 2010.
\newblock \doi{10.1109/TRO.2010.2042754}.

\bibitem[Porta et~al.(2006)Porta, Vlassis, Spaan, and Poupart]{porta2006point}
J.M. Porta, N.~Vlassis, M.T.J. Spaan, and P.~Poupart.
\newblock Point-based value iteration for continuous {POMDP}s.
\newblock \emph{Journal of Machine Learning Research}, 7:\penalty0 2329--2367,
  2006.

\bibitem[Ross et~al.(2008)Ross, Pineau, Paquet, and Chaib-Draa]{ross2008online}
S.~Ross, J.~Pineau, S.~Paquet, and B.~Chaib-Draa.
\newblock Online planning algorithms for {POMDP}s.
\newblock \emph{Journal of Artificial Intelligence Research}, 2008.
\newblock \doi{10.1613/jair.2567}.

\bibitem[Salisbury et~al.(1988)Salisbury, Townsend, Eberman, and
  DiPietro]{salisbury1988preliminary}
K.~Salisbury, W.~Townsend, B.~Eberman, and D.~DiPietro.
\newblock Preliminary design of a whole-arm manipulation system ({WAMS}).
\newblock In \emph{{IEEE} International Conference on Robotics and Automation},
  1988.
\newblock \doi{10.1109/ROBOT.1988.12057}.

\bibitem[Silver and Veness(2010)]{silver2010monte}
D.~Silver and J.~Veness.
\newblock {M}onte-{C}arlo planning in large {POMDP}s.
\newblock In \emph{Advances in Neural Information Processing Systems}, 2010.

\bibitem[Smallwood and Sondik(1973)]{smallwood1973optimal}
R.D. Smallwood and E.J. Sondik.
\newblock The optimal control of partially observable {M}arkov processes over a
  finite horizon.
\newblock \emph{Operations Research}, 21\penalty0 (5):\penalty0 1071--1088,
  1973.
\newblock \doi{10.1287/opre.21.5.1071}.

\bibitem[Somani et~al.(2013)Somani, Ye, Hsu, and Lee]{somani2013despot}
A.~Somani, N.~Ye, D.~Hsu, and W.S. Lee.
\newblock {DESPOT}: Online {POMDP} planning with regularization.
\newblock In \emph{Advances in Neural Information Processing Systems}, 2013.

\bibitem[Srinivasa et~al.(2012)Srinivasa, Berenson, Cakmak, Collet, Dogar,
  Dragan, Knepper, Niemueller, Strabala, and Vande~Weghe]{srinivasa2012herb}
S.S. Srinivasa, D.~Berenson, M.~Cakmak, A.~Collet, M.R. Dogar, A.D. Dragan,
  R.A. Knepper, T.~Niemueller, K.~Strabala, and M.~Vande~Weghe.
\newblock {HERB} 2.0: Lessons learned from developing a mobile manipulator for
  the home.
\newblock \emph{Proceedings of the {IEEE}}, 100\penalty0 (8):\penalty0 1--19,
  2012.

\bibitem[Stulp et~al.(2011)Stulp, Theodorou, Buchli, and
  Schaal]{stulp2011learning}
F.~Stulp, E.~Theodorou, J.~Buchli, and S.~Schaal.
\newblock Learning to grasp under uncertainty.
\newblock In \emph{{IEEE} International Conference on Robotics and Automation},
  pages 5703--5708, 2011.
\newblock \doi{10.1109/ICRA.2011.5979644}.

\bibitem[Townsend(2000)]{townsend2000barretthand}
W.~Townsend.
\newblock The {BarrettHand} grasper--programmably flexible part handling and
  assembly.
\newblock \emph{Industrial Robot: An International Journal}, 27\penalty0
  (3):\penalty0 181--188, 2000.
\newblock \doi{10.1108/01439910010371597}.

\bibitem[Zhang and Chen(2000)]{zhang2000control}
H.~Zhang and N.N. Chen.
\newblock Control of contact via tactile sensing.
\newblock \emph{{IEEE} Transactions on Robotics and Automation}, 16\penalty0
  (5):\penalty0 482--495, 2000.
\newblock \doi{10.1109/70.880799}.

\end{thebibliography}
